\documentclass[english,a4paper,12pt]{article}

\usepackage{fullpage}

\usepackage{amsxtra, amsfonts, amssymb, amstext, amsmath, mathtools}
\usepackage{amsthm}
\usepackage{booktabs}
\usepackage{nicefrac}
\usepackage{xspace}
\usepackage{url}
\usepackage{graphics,color}
\usepackage[algo2e,ruled,vlined,linesnumbered]{algorithm2e}
\usepackage{wrapfig}
\usepackage{lmodern}
\usepackage{array}
\newcolumntype{?}{!{\vrule width 1.5pt}}

\newtheorem{theorem}{Theorem}
\newtheorem{lemma}[theorem]{Lemma}
\newtheorem{corollary}[theorem]{Corollary}

\newcommand{\oplea}{\mbox{${(1+\lambda)}$~EA}\xspace}

\newcommand{\NSGA}{NSGA\nobreakdash-II\xspace}

\newcommand{\jump}{\textsc{Jump}\xspace}
\newcommand{\oneminmax}{\textsc{OneMinMax}\xspace}

\newcommand{\lotz}{\textsc{LOTZ}\xspace}
\newcommand{\ojzj}{\textsc{OneJumpZeroJump}\xspace}
\newcommand{\onejumpzerojump}{\textsc{OneJumpZeroJump$_{n,k}$}\xspace}
\newcommand{\jumpnk}{\textsc{Jump$_{n,k}$}\xspace}

\newcommand{\omm}{\oneminmax}

\DeclareMathOperator{\cDis}{cDis}

\let\originalleft\left
\let\originalright\right
\renewcommand{\left}{\mathopen{}\mathclose\bgroup\originalleft}
\renewcommand{\right}{\aftergroup\egroup\originalright}

\usepackage{hyperref}
\begin{document}
\title{From Understanding the Population Dynamics of the NSGA-II \\ to the First Proven Lower Bounds}

\author{Benjamin Doerr \and Zhongdi Qu}
\date{}

\maketitle

\begin{abstract}
  Due to the more complicated population dynamics of the NSGA-II, none of the existing runtime guarantees for this algorithm is accompanied by a non-trivial lower bound. Via a first mathematical understanding of the population dynamics of the NSGA-II, that is, by estimating the expected number of individuals having a certain objective value, we prove that the NSGA-II with suitable population size needs $\Omega(Nn\log n)$ function evaluations to find the Pareto front of the \oneminmax problem and $\Omega(Nn^k)$  evaluations on the \ojzj problem with jump size~$k$. These bounds are asymptotically tight (that is, they match previously shown upper bounds) and show that the NSGA-II here does not even in terms of the parallel runtime (number of iterations) profit from larger population sizes. For the \ojzj problem and when the same sorting is used for the computation of the crowding distance contributions of the two objectives, we even obtain a runtime estimate that is tight including the leading constant. 
\end{abstract}

{\sloppy 
\section{Introduction}

Many real-world problems have several, often conflicting objectives. For such \emph{multi-objective} optimization problems, it is hard to compute a single solution. Instead, one usually computes a set of incomparable, interesting solutions from which a decision maker can select the most preferable one. Due to their population-based nature, evolutionary algorithms (EAs) are well suited for such problems, and in fact, are intensively used in multi-objective optimization. 

The most accepted multi-objective evolutionary algorithm (MOEA) in practice~\cite{ZhouQLZSZ11} is the \emph{non-dominated sorting genetic algorithm~II} (\NSGA) proposed in~\cite{DebPAM02}. It uses a fixed population size~$N$, generates $N$ new solutions per iteration, and selects the next population according to the non-dominated sorting of the combined parent and offspring population and the crowding distance. Due to this complex structure, for a long time no mathematical runtime analyses existed for this algorithm. However, in 2022 four such works appeared, all greatly enhancing our understanding of how this algorithm works, but also detecting weaknesses and proposing promising remedies~\cite{ZhengLD22, ZhengD22gecco, BianQ22ppsn, DoerrQ22ppsn}. 

Interestingly, and different from the previous runtime analyses of other MOEAs, none of these works proved a non-trivial lower bound on the runtime. Such bounds are important since only by comparing upper and lower bounds for different algorithms can one declare one algorithm superior to another. Such bounds are also necessary to have tight runtime estimates, from which information about optimal parameter values can be obtained.  

The lack of lower bounds for the \NSGA, naturally, is caused by the more complicated population dynamics of this complex algorithm. While for algorithms like the SEMO or global SEMO predominantly analyzed in MOEA theory, it follows right from the definition of the algorithm that there can be at most one individual per objective value, such a structural information does not exist for the \NSGA.

In this work, we gain a first deeper understanding of the population dynamics of the \NSGA (more precisely, the mutation-based version regarded in almost all previous runtime results for the \NSGA). For the optimization of the \oneminmax and the \ojzj benchmark, we prove that also for relatively large population sizes, only a constant number of individuals exists on the outer positions of the Pareto front. This information allows us to prove upper bounds on the speed with which the Pareto front is explored, and finally yields lower bounds on the runtime of the \NSGA on these two benchmarks. 

More specifically, we prove the following lower bound for the \ojzj benchmark (we do not discuss here in detail the result for \omm and refer instead to Theorem~\ref{thm:omm}). Let~$N$ denote the population size of the \NSGA, $n$ denote the problem size (length of the bit-string encoding), and $k$ the gap parameter of the \ojzj problem. If $N$ is at least $4$ times the size $n-2k+3$ of the Pareto front and $N = o(n^2 / k^2)$, then the time to compute the Pareto front of the \ojzj problem is at least $(\frac{3(e-1)}{8} - o(1)) N n^k$ fitness evaluations or, equivalently, $(\frac{3(e-1)}{8} - o(1)) n^k$ iterations. This result shows that the upper bound of $O(Nn^k)$ fitness evaluations or $O(n^k)$ iterations proven in~\cite{DoerrQ22ppsn} is asymptotically tight for broad ranges of the parameters. In particular, this shows that there is no advantage in using a population size larger than the smallest admissible one, not even when taking the number of iterations as the performance measure. This is very different from the single-objective world, where, for example, the runtime of the \oplea on the single-objective \jump problem with jump size~$k$ is easily seen to be $\Omega(n^k / \lambda)$ and $O(n^k / \lambda + n \log n)$ iterations, hence for $\lambda = o(n^{k-1} / \log(n))$ the number of iterations reduces with growing value of $\lambda$ and the number of fitness evaluations does not change (when ignoring lower-order terms). This comparison suggests that the choice of the population size might be more critical for the \NSGA than for single-objective EAs. 

For the variant of the \NSGA which uses the same sorting to compute the crowding distance contribution of both objectives (which is a natural choice for two objectives), we can even determine the runtime precise apart from lower order terms. To this aim, we also exploit our new understanding of the population dynamics to prove a tighter upper bound on the runtime. We shall not exploit this further in this work, but we note that such tight analyses are the prerequisite for optimizing parameters, here for example the mutation rate. To the best of our knowledge, this is only the second runtime analysis of a MOEA that determines the leading constant of the runtime (the other one being the analysis of the synthetic GSEMO algorithm on the \ojzj benchmark).

Overall, this work constitutes a first step towards understanding the population dynamics of the \NSGA. We exploit this to prove the first asymptotically tight lower bounds. In a non-trivial special case, we even determine a runtime precise apart from lower-order terms. These results already give some information on the optimal parameter values, and we are optimistic that our methods can lead to more insights about the right parameter choices of the \NSGA.

\section{Previous Works}\label{sec:previous}

For a general introduction to multi-objective optimization via evolutionary algorithms, including the most prominent algorithm \NSGA (47000 citations on Google scholar), we refer to~\cite{ZhouQLZSZ11}. This work is concerned with the mathematical runtime analysis of a MOEA, which is a subarea of the broader research area of runtime analyses for randomized search heuristics~\cite{AugerD11}. 

The first runtime analyses of MOEAs date back to the early 2000s~\cite{LaumannsTZWD02,Giel03,Thierens03} and regarded the artificial SEMO and GSEMO algorithms, which are still the most regarded algorithms in MOEA theory (see, e.g.,~\cite{BianQT18ijcaigeneral, QianYTYZ19, QianLZ22} for some recent works). Some time later, the first analyses of the more realistic MOEAs SIBEA~\cite{BrockhoffFN08,NguyenSN15,DoerrGN16} and the MOEA/D~\cite{LiZZZ16,HuangZCH19,HuangZ20,HuangZLL21} followed. Very recently, the first runtime analysis of the \NSGA appeared~\cite{ZhengLD22}, which was quickly followed up by further runtime analyses of this algorithm.

In~\cite{ZhengLD22}, it was proven that the \NSGA with population size~$N$ can efficiently optimize the classic \omm and \lotz benchmarks if $N$ is at least four times the size of the Pareto front. A population size equal to the size of the Pareto front does not suffice. In this case, the \NSGA loses desired solutions often enough so that the \NSGA covers only a constant fraction of the Pareto front for at least an exponential time. However, with smaller population sizes, the \NSGA can still compute good approximations to the Pareto front, as proven in~\cite{ZhengD22gecco} again for the \omm benchmark. The first runtime analysis on a benchmark with multimodal objectives~\cite{DoerrQ22ppsn} showed that the \NSGA, again with population size $N$ at least four times the size of the Pareto front, computes the Pareto front of the \ojzj benchmark with jump parameter $k \in [2..n/4]$ in expected time at most $O(N n^k)$. These three works regard a version of the \NSGA without crossover. In~\cite{BianQ22ppsn}, besides other results, the original \NSGA with crossover is regarded. However, no better runtime guarantees are proven for this algorithm.

For none of the runtime guarantees proven in these works, a matching (or at least non-trivial) lower bound was shown. The apparent reason, spelled out explicitly in the conclusion of~\cite{DoerrQ22ppsn}, is the lack of understanding of the population dynamics for this algorithm. We note that this problem is very present even for the simpler SEMO/GSEMO algorithm despite the fact that here the population is much more restricted. In particular, the strict selection mechanism of these algorithms ensures that for each objective value there is at most one individual in the population. Nevertheless, also a decent number of runtime results for these algorithms do not have a matching lower bound. For the SEMO algorithm, which uses one-bit mutation, a lower bound matching the $O(n^2 \log n)$ upper bound on \omm of~\cite{GielL10} was shown ten years later in~\cite{OsunaGNS20}. For the GSEMO, using bit-wise mutation instead, no lower bound is known for the \omm benchmark. Similarly, for the \lotz benchmark a tight $\Theta(n^3)$ runtime of the SEMO was proven in~\cite{LaumannsTZWD02} already. The same upper bound was proven in that work for the GSEMO, but the only lower bound~\cite{DoerrKV13} for this problem is valid only for an unrealistically small mutation rate. Only for the \ojzj benchmark, a tight bound of $\Theta(n^k)$ was proven also for the GSEMO~\cite{DoerrZ21aaai}, clearly profiting from the fact that the population can contain at most one individual on the local optimum of the objectives.

\section{Preliminaries}\label{sec:prelim}
\subsection{The NSGA-II Algorithm}\label{sec:prelim-alg}

In the interest of brevity, we only give a brief overview of the algorithm here and refer to \cite{DebPAM02} for a more detailed description of the general algorithm and to \cite{ZhengLD22} for more details on the particular version of the NSGA\nobreakdash-II we regard.

The NSGA-II uses two metrics, rank and crowding distance, to completely order any population. The ranks are defined recursively based on the dominance relation. All non-dominated individuals have rank~1. Then, given that the individuals of ranks $1, \dots, k$ are defined, the individuals of rank $k+1$ are those not dominated except by individuals of rank $k$ or smaller. This defines a partition of the population into sets $F_1$, $F_2$,\dots such that $F_i$ contains all individuals with rank $i$. Clearly, individuals with lower ranks are preferred. The crowding distance, denoted by $\cDis(x)$ for an individual~$x$, is used to compare individuals of the same rank. To compute the crowding distances of individuals of rank $i$ with respect to a given objective function $f_j$, we first sort the individuals in ascending order according to their $f_j$ objective values. The first and last individuals in the sorted list have infinite crowding distance. For the other individuals, their crowding distance is the difference between the objective values of its left and right neighbors in the sorted list, normalized by the difference of the minimum and maximum values. The final crowding distance of an individual is the sum of its crowding distances with respect to each objective function. Among individuals of the same rank, the ones with higher crowding distances are preferred.

The algorithm starts with a random initialization of a parent population of size~$N$. In each iteration, $N$ children are generated from the parent population via a variation operator, and $N$ best individuals among the combined parent and children population survive to the next generation based on their ranks and, as a tie-breaker, the crowding distance. At each iteration, the critical rank $i^*$ is the rank such that if we take all individuals of ranks smaller than $i^*$, the total number of individuals will be less than or equal to $N$, but if we also take all individuals of rank $i^*$, the total number of individuals will be more than~$N$. Thus, all individuals of rank smaller than $i^*$ survive to the next generation, and for individuals of rank $i^*$, we take the individuals with the highest crowding distance, breaking ties randomly, so that in total exactly $N$ individuals are kept. In practice, the algorithm is run until some stopping criterion is met. In our mathematical analysis, we are interested in how long it takes until the full Pareto front is covered by the population if the algorithm is not stopped earlier. For that reason, we do not specify a termination criterion. 

In our analysis, for simplicity we assume that in each iteration every parent produces one child through bit-wise mutation, i.e., mutating each bit independently with probability~$\frac{1}{n}$. Our analysis also holds for uniform selection, where $N$ times a parent is selected independently at random, since also here each individual is selected as parent once in expectation, and our proofs only rely on the expected number of times a parent is selected. When selecting parents via binary tournaments the expected number of times a parent is selected is at most two (which is the expected number of times it participates in a tournament). This estimate would change the population dynamics by constant factors. For that reason, we are optimistic that our methods apply also to this type of selection, but we do not discuss this question in more detail.

For any generation $t$ of a run of the algorithm, we use $P_t$ to denote the parent population and $R_t$ to denote the combined parent and offspring population. 


\subsection{The \ojzj Benchmark}
Let $n\in\mathbb{N}$ and $k=[2..n/4]$. The function $\onejumpzerojump = (f_1, f_2):\{0,1\}^n\rightarrow\mathbb{R}^2$, proposed by \cite{DoerrZ21aaai}, is defined by
\[f_1(x) = \begin{cases}
    k+|x|_1, & \text{if }|x|_1 \leq n-k\text{ or } x=1^n,\\
    n-|x|_1, & \text{else};
    \end{cases}\]
\[f_2(x) = \begin{cases}
    k+|x|_0, & \text{if }|x|_0 \leq n-k\text{ or } x=0^n,\\
    n-|x|_0, & \text{else},
    \end{cases}\]
where $|x|_1$ denotes the number of bits of $x$ that are 1 and $|x|_0$ denotes the number of bits of $x$ that are 0. The aim is to maximize both $f_1$ and $f_2$, two multimodal objectives. The first objective is the classical $\jumpnk$ function. It has a valley of low fitness around its optimum, which can be crossed only by flipping the $k$ correct bits, if no solutions of lower fitness are accepted. The second objective is isomorphic to the first, with the roles of zeroes and ones exchanged.

According to Theorem~$2$ of \cite{DoerrZ21aaai}, the Pareto set of this benchmark is $S^*=\{x \in \{0,1\}^n \mid|x|_1 = [k..n-k]\cup\{0, n\}\}$, and the Pareto front $F^*=f(S^*)$ is $\{(a, 2k+n-a)\mid a\in[2k..n]\cup\{k, n+k\}\}$, making the size of the front $n-2k+3$. We define the inner part of the Pareto set by $S_{I}^*=\{x\mid|x|_{1}\in [k..n-k]\}$, and the inner part of the Pareto front by $F_{I}^*=f(S_{I}^*)=\{(a, 2k+n-a)\mid a\in[2k..n]\}$. \cite{DoerrQ22ppsn} showed that when using a population of size $N\geq 4(n-2k+3)$ to optimize this benchmark, the NSGA-II algorithm never loses a Pareto-optimal solution once found. Moreover, $O(n^k)$ iterations are needed in expectation. 

\subsection{The \oneminmax Benchmark}
Let $n\in \mathbb{N}$. The function $\oneminmax=(f_1, f_2):\{0, 1\}^n\rightarrow \mathbb{R}$, proposed by \cite{GielL10}, is defined by
\[f(x)=(f_1(x), f_2(x))=\left(n-\sum_i^n x_i, \sum_i^n x_i\right).\]
The aim is to maximize both objectives.

For this benchmark, any solution is Pareto-optimal and the Pareto front $F^*=\{(0, n), (1, n-1), \dots ,(n,0)\}$. Hence $|F^*|=n+1$. \cite{ZhengLD22} showed that when using a population of size $N\geq4(n+1)$ to optimize the benchmark, the NSGA-II algorithm never loses a Pareto-optimal solution once found. Moreover, in expectation $O(n\log n)$ iterations are needed. 

\section{Lower Bound on the Runtime of the NSGA-II on \ojzj}\label{sec:low_bound_ojzj}
In this section, we give a lower bound on the runtime of the NSGA-II algorithm on the \ojzj benchmark. We use $X_{P_t}^i$ to denote the number of individuals with $n-k-i$ 1-bits in $P_t$ and $X_{R_t}^i$ to denote that in $R_t$.

We first show that with probability arbitrarily close to $1$, we have that $P_t\subseteq S^*$ for any $t$ so that  the analyses that follow do not need to consider gap individuals (those with between $1$ and $k-1$ zeroes or ones) as parents.

\begin{lemma}\label{lem:all_in} Consider the NSGA-II algorithm optimizing the \onejumpzerojump benchmark for $2\leq k \leq \frac{n}{4}$ with population size $N=c(n-2k+3)$ for $c=o(n)$. The probability that $P_0 \subseteq S_{I}^*$ is $1-o(1)$. Moreover, if $P_0 \subseteq S_{I}^*$, then for any generation $t$, we have $P_t \subseteq S^*$.
\end{lemma}

\begin{proof}
For a randomly initialized individual in $P_0$, we have $\mathbb{E}[|x|_1] = \frac{n}{2}$. Since $k \leq \frac{n}{4}$, by the additive Chernoff bound, with probability at most $2e^{-\frac{2}{n}(\frac{n}{4})^2}=2e^{-\frac{n}{8}}$, the initial individual $x$ has $|x|_1 < k$ or $|x|_1 > n-k$. Then by the union bound, $P_0 \subseteq S_{I}^*$ with probability at least $1-2Ne^{-\frac{n}{8}}$, which is $1-o(1)$ since $c=o(n)$.

Consider a run of the algorithm where $P_0 \subseteq S_{I}^*$. We show by induction that $P_t \subseteq S^*$. The base case is obviously $P_0 \subseteq S_{I}^*\subseteq S^*$. Suppose $P_{t}\subseteq S^*$ by the induction hypothesis. Since any individual in $S^*$ dominates any individual not in $S^*$, the rank-1 individuals in $R_t$ are exactly those in $S^*$. Since there are at least $|P_t|=N$ individuals in $S^*$ in $R_t$, no individuals not in $S^*$ can survive. Therefore $P_{t+1}\subseteq S^*$.  
\end{proof}

Now we give an upper bound on the probability for any individual to obtain more $1$-bits through bit-wise mutation.

\begin{lemma}\label{lem:going_up} 
  Let $n, u, v\in \mathbb{N}$ with $n\geq 2$, $u, v\geq 1$, and $u+v \leq n$. Suppose $x \in \{0, 1\}^n$ and $|x|_1 \leq v$. Denote the result of applying bit-wise mutation to $x$ by $x'$. Then \[\Pr[|x'|_1 = u+v] \leq \left(\frac{n-v}{n}\right)^u.\]
\end{lemma}

\begin{proof}
Suppose an individual $y$ has $v$ bits of $1$ and the result of applying bit-wise mutation to $y$ is $y'$. Then $\Pr[|y'|_1 = u+v] \leq \binom{n-v}{u}(\frac{1}{n})^u\leq (\frac{n-v}{n})^u$ since $u$ of the $n-v$ 0-bits of $y$ have to be flipped. Therefore, to prove the claim, we show that $\Pr[|x'|_1 = u+v] \leq \Pr[|y'|_1 = u+v]$.

The case for $|x|_1 = |y|_1$ is obvious. Then suppose $|x|_1 = v-1$. We have $\Pr[|y'|_1 = u+v] = $
\begin{equation}\label{eq:1}\sum_{i=u}^{\min\{u+v, n-v\}}\binom{n-v}{i}\binom{v}{i-u}\left(\frac{1}{n}\right)^{2i-u}\left(1-\frac{1}{n}\right)^{n-2i+u}\end{equation}
and $\Pr[|x'|_1 = u+v] =$
\begin{equation}\label{eq:2}\begin{aligned}
\sum_{i=u}^{\min\{u+v-1, n-v\}}(&\binom{n-v+1}{i+1}\binom{v-1}{i-u}\\&\left(\frac{1}{n}\right)^{2i-u+1}\left(1-\frac{1}{n}\right)^{n-2i+u-1})\end{aligned}
\end{equation} 
Dividing the summands of equation \eqref{eq:1} by those of \eqref{eq:2} one by one, we have the quotient $\frac{{n-v\choose i}{v\choose {i-u}}(\frac{1}{n})^{2i-u}(1-\frac{1}{n})^{n-2i+u}}{{n-v+1\choose i+1}{v-1\choose {i-u}}(\frac{1}{n})^{2i-u+1}(1-\frac{1}{n})^{n-2i+u-1}} = \frac{v(i+1)(n-1)}{(v-i+u)(n-v+1)}$, which increases in $i$. So the quotient is minimized when $i=u$, making it $\frac{(u+1)(n-1)}{n-v+1}$, which in turn is minimized when $v=u=1$, making it $\frac{2(n-1)}{n} \geq 1$ for $n\geq 2$. Since the summands in equations \eqref{eq:1} and \eqref{eq:2} both start at $i=u$ and there are at least as many summands in \eqref{eq:1} as in \eqref{eq:2}, we have $\Pr[|x'|_1 = u+v] \leq \Pr[|y'|_1 = u+v]$. The cases for $|x|_1 < u-1$ follow by induction based on the case for $|x|_1 = u-1$.
\end{proof}

\begin{corollary}\label{cor:flip_positive}
Let $n\in \mathbb{N}$ with $n\geq 2$ and let $v\in [1..n]$. Suppose $x \in \{0, 1\}^n$ and $|x|_1 \leq v$. Denote the result of applying bit-wise mutation to $x$ by $x'$. Then $\Pr[|x'|_1=v \land x' \neq x] \leq \frac{n-v+1}{n}$.
\end{corollary}
\begin{proof}
Consider the case where $|x|_1=v$. Since $x'\neq x$, at least one of the $n-v$ 0-bits of $x$ has to be flipped, and the probability that happens is at most $\frac{n-v}{n} < \frac{n-v+1}{n}$.

Consider the case where $|x|_1 \le v-1$. By Lemma \ref{lem:going_up}, $\Pr[|x|_1'=v] \leq \frac{n-v+1}{n}$. 
\end{proof}

Since we already know from \cite{DoerrQ22ppsn} that, for any objective value on the Pareto front, there are at most $4$ individuals with that objective value and positive crowding distance, and they all survive to the generation that follows, to further understand the population dynamics on the front, it is crucial to analyze what happens to the individuals with zero crowding distance. In the following lemma, we show that their survival probability is less than $\frac{1}{2}+o(1)$.  
\begin{lemma}\label{lem:zero_survives} Consider the NSGA-II algorithm optimizing the \onejumpzerojump benchmark with the population size $N=c(n-2k+3)$ for some $c\geq 4$ such that $ck^2=o(n)$. Consider a generation $t$ of a run of the algorithm where $P_0\subseteq S_I^*$. Suppose $\mathbb{E}[X_{P_t}^0]=O(ck)$ and $\mathbb{E}[X_{P_t}^{n-2k}]=O(ck)$. For a rank-1 individual $x\in R_t$ that has zero crowding distance, the probability that $x\in P_{t+1}$ is less than $\frac{1}{2}+o(1)$.
\end{lemma}

\begin{proof}
Let $F_{>1}$ denote the individuals in $R_t$ with ranks greater than $1$. Since $P_0\subseteq S_I^*$, by Lemma~\ref{lem:all_in}, $P_t \subseteq S^*$. Then all the individuals in $F_{>1}$ are created through mutation of individuals in $P_t$. By Lemma~\ref{lem:going_up}, for an individual with less than $n-k$ bits of $1$ to create an individual with more than $n-k$ bits of $1$, the probability is at most $(\frac{k+1}{n})^2$, and for an individual with $n-k$ bits of $1$ to create an individual with more than $n-k$ bits of $1$, the probability is at most $\frac{k}{n}$. Symmetrically, for an individual with less than $n-k$ bits of $0$ to create an individual with more than $n-k$ bits of $0$, the probability is at most $(\frac{k+1}{n})^2$, and for an individual with $n-k$ bits of $0$ to create an individual with more than $n-k$ bits of $0$, the probability is at most $\frac{k}{n}$. Therefore $\mathbb{E}[|F_{>1}|]\leq (\frac{k+1}{n})^2c(n-2k+3)+\frac{k}{n}\mathbb{E}[X_{P_t}^0]+\frac{k}{n}\mathbb{E}[X_{P_t}^{n-2k}]=o(1)$ for $ck^2=o(n)$, $\mathbb{E}[X_{P_t}^0]=O(ck)$ and $\mathbb{E}[X_{P_t}^{n-2k}]=O(ck)$. By Markov's inequality, $\Pr(|F_{>1}|\geq 1)\leq \frac{\mathbb{E}[|F_{>1}|]}{1}=o(1)$.

Let $F_{1}^*$ the rank-$1$ individuals in $R_t$ with positive crowding distances. So $|R_t|=2N$ and there are $2N-|F_{1}^*|-|F_{>1}|$ individuals in $R_t$ with rank $1$ and zero crowding distance. Since $N=c(n-2k+3)$, for some $c\geq 4$, by Lemma 1 of \cite{DoerrQ22ppsn}, all individuals in $F_{1}^*$ will survive. So among the rank-$1$ individuals with zero crowding distance, $N-|F_{1}^*|$ survive to the next generation if $|F_{1}^*|\leq N$. Hence the probability that a rank-$1$ individual $x$ with zero crowding distance survives is \[\begin{aligned}\frac{N-|F_{1}^*|}{2N-|F_{1}^*|}&\Pr[|F_{>1}|=0]\\& +\Pr[x \text{ survives}||F_{>1}|\geq 1]]\Pr[|F_{>1}|\geq 1].\end{aligned}\]
Since $\Pr[|F_{>1}|=0]\leq 1$ and $\Pr[x \text{ survives}||F_{>1}|\geq 1]]\leq 1$, we have that the probability that $x$ survives is at most $\frac{1}{2}+o(1)$.
\end{proof}

\begin{corollary}\label{cor:develop}
Consider the NSGA-II algorithm optimizing the \onejumpzerojump benchmark with the population size $N=c(n-2k+3)$ for some $c\geq 4$ such that $ck^2=o(n)$. Consider a generation $t$ of a run of the algorithm where $P_0\subseteq S_I^*$. Suppose $\mathbb{E}[X_{P_t}^0]=O(ck)$ and $\mathbb{E}[X_{P_t}^{n-2k}]=O(ck)$. Then for any $i\in[0..n-2k]$, we have $\mathbb{E}[X^i_{P_{t+1}}] \leq (\frac{1}{2}+o(1))\mathbb{E}[X^i_{R_t}]+2$.  
\end{corollary}
\begin{proof}
Among the individuals with $i$ 1-bits in $R_t$, let $X_{>0}$ denote the number of individuals with positive crowding distance, $X_{=0}$ denote the number of individuals with zero crowding distance, and $X_{=0}^*$ denote the number of individuals with zero crowding distance that survive to the next generation. Then $\mathbb{E}[X^i_{P_{t+1}}] \leq \mathbb{E}[X_{>0}] + \mathbb{E}[X_{=0}^*]$. By Lemma~\ref{lem:zero_survives}, $\mathbb{E}[X_{=0}^*] \leq (\tfrac{1}{2}+o(1))\mathbb{E}[X_{=0}]$, so \[\begin{aligned}
\mathbb{E}[X^i_{P_{t+1}}] & \leq \mathbb{E}[X_{>0}] + (\tfrac{1}{2}+o(1))\mathbb{E}[X_{=0}] \\ &=(\frac{1}{2}+o(1))\mathbb{E}[X^i_{R_t}]+(\tfrac{1}{2}-o(1))\mathbb{E}[X_{>0}].
\end{aligned}\] By Lemma~1 of \cite{DoerrQ22ppsn}, $\mathbb{E}[X_{>0}]\leq 4$, so  $\mathbb{E}[X_{P^i_{t+1}}] \leq (\frac{1}{2}+o(1))\mathbb{E}[X_{R^i_t}]+2$.
\end{proof}

Now, we can start to estimate $\mathbb{E}[X_{P_t}^i]$ for $i\in[0..n-2k]$.
\begin{lemma}\label{lem:bound} Consider the NSGA-II algorithm optimizing the \onejumpzerojump benchmark with the population size $N=c(n-2k+3)$ for some $c\geq 4$ such that $ck^2=o(n)$. Suppose $P_0 \subseteq S_{I}^*$ and $i\in[0..n-2k]$. Then if $1^n\notin P_t$, we have $\mathbb{E}[X_{P_t}^i] \leq c_i$ for $c_i = \frac{e}{e-1}(c(k+i+1)+ \sum_{j=0}^{i-1}c_j + 4)+o(ck)$. Similarly, if $0^n \notin P_t$, we have $\mathbb{E}[X_{P_t}^{n-2k-i}] \leq c_{n-2k-i}$ for $c_{n-2k-i} = \frac{e}{e-1}(c(k+i+1)+ \sum_{j=0}^{i-1}c_{n-2k-j} +4)+o(ck)$. 
\end{lemma}

\begin{proof}
We prove the result for the case where the all-ones string has not been found since the other case is symmetrical.

Let $Y^i$ denote the number of individuals with $n-k-i$ $1$-bits in $P_t$ for which no bits are flipped during mutation, and let $Z^i$ denote the number of individuals in $P_t$ for which a positive number of bits are flipped and the resulting children have $n-k-i$ 1-bits. We first prove by induction that $\mathbb{E}[X_{P_t}^0] \leq c_0$ for $c_0 = \frac{e}{e-1}(c(k+1)+4)+o(ck)=O(ck)$. 

For the base case, consider the random initialization of $P_0$. The probability that exactly $k$ among $n$ bits are $0$ is less than the probability that at most $k < \frac{n}{4}$ bits are $0$, which is at most $e^{-\frac{n}{8}}$. Hence, $\mathbb{E}[X_{P_0}^0] < c(n-2k+3)e^{-\frac{n}{8}} < cne^{-\frac{n}{8}} < 3c < c_0$.

For the induction, assume by the induction hypothesis that $\mathbb{E}[X_{P_{t}}^0] \leq c_0$ and we will show that $\mathbb{E}[X_{P_{t+1}}^0] \leq c_0$. Clearly, $\mathbb{E}[X_{R_t}^0] = \mathbb{E}[X_{P_t}^0] + \mathbb{E}[Y^0] + \mathbb{{E}}[Z^0]$. By the induction hypothesis $\mathbb{E}[X_{P_t}^0] \leq c_0$. By definition, $\mathbb{{E}}[Y^0] = \mathbb{{E}}[(1-\frac{1}{n})^nX_{P_t}^0] \leq \frac{1}{e}\mathbb{{E}}[X_{P_t}^0]=\frac{c_0}{e}$. Since $1^n\notin P_t$ and $P_0 \subseteq S_{I}^*$, by Lemma~\ref{lem:all_in}, there is no individual with more than $n-k$ 1-bits. Then by Corollary~\ref{cor:flip_positive}, for any individual to have a positive number of bits flipped and produce an individual with $n-k$ 1-bits, the probability is at most $\frac{k+1}{n}$. So $\mathbb{E}[Z^0]\leq c(n-2k+3)\frac{k+1}{n}\leq c(k+1)$. Together, $\mathbb{{E}}[X_{R_t}^0]\leq (1+\frac{1}{e})c_0+c(k+1)$. Then by Corollary~\ref{cor:develop},  $\mathbb{E}[X_{P_{t+1}}^0]\leq (\frac{1}{2}+o(1))((1+\frac{1}{e})c_0+c(k+1)) + 2 = \frac{e}{e-1}(c(k+1)+4)+o(ck) = c_0$.

The same arguments can be applied to estimate the expected number of individuals with $n-k-1$ 1-bits. We have $\mathbb{E}[X_{R_t}^1] = \mathbb{E}[X_{P_t}^1] + \mathbb{E}[Y^1] + \mathbb{{E}}[Z^1]$. We assume by the induction hypothesis $\mathbb{E}[X_{P_{t}}^1] \leq c_1$ for $c_1 = \frac{e}{e-1}(c(k+2)+ c_0 +4)+o(ck)$. Similarly as before, $\mathbb{E}[Y^1]\leq \frac{c_1}{e}$. Moreover, there are no individuals with more than $n-k$ bits of $1$ by Lemma~\ref{lem:all_in}. So to bound $\mathbb{{E}}[Z^1]$, consider separately the cases where i) the parent has at most $n-k-1$ 1-bits and ii) the parent has $n-k$ 1-bits. By Corollary~\ref{cor:flip_positive}, for case i), the probability that the child has $n-k-1$ 1-bits is at most $\frac{k+2}{n}$. We trivially bound the probability for case ii) by $1$. Therefore, $\mathbb{{E}}[Z^1] \leq c(n-2k+3)\frac{k+2}{n} + c_0\leq c(k+2) + c_0$. Then $\mathbb{E}[X_{R_t}^1] \leq(1+\frac{1}{e})c_1 + c(k+2) + c_0$. Hence, by Corollary~\ref{cor:develop}, \[\begin{aligned}\mathbb{E}[X_{P_{t+1}}^1] & \leq (\tfrac{1}{2}+o(1))((1+\tfrac{1}{e})c_1 + c(k+2) + c_0) + 2 \\&= \tfrac{e}{e-1}(c(k+2)+ c_0 +4)+o(ck) = c_1.\end{aligned}\]

Continuing this way and letting $c_i$ denote the upper bound on the expected number of individuals with $n-k-i$ number of 1-bits for $0 \leq i \leq n-2k$, we have \[c_i = \frac{e}{e-1}(c(k+i+1)+ \sum_{j=0}^{i-1}c_j + 4)+o(ck).\qedhere\]
\end{proof}

With the bound on $\mathbb{E}[X_{P_{t}}^1]$ found in Lemma~\ref{lem:bound}, we can now prove a sharper bound on $\mathbb{E}[X_{P_{t}}^0]$.
\begin{corollary}\label{cor:constant}
Consider a generation $t$ of the NSGA-II algorithm optimizing the \onejumpzerojump benchmark with the population size $N=c(n-2k+3)$ for some $c\geq 4$ such that $ck^2=o(n)$. Suppose $P_0 \subseteq S_{I}^*$. If $1^n\notin P_t$, then $\mathbb{E}[X^0_t]\le \frac{4e}{e-1}+o(1)$. Similarly, if $0^n\notin P_t$, then $\mathbb{E}[X^{n-2k}_t]\le \frac{4e}{e-1}+o(1)$.
\end{corollary}
\begin{proof}
We prove the result for $\mathbb{E}[X^0_t]$ since the other case is symmetrical.

Using the same notations as in the proof of Lemma~\ref{lem:bound}, with the bound on $\mathbb{E}[X_{P_{t}}^1]$ found in Lemma~\ref{lem:bound}, we can prove a sharper bound on $\mathbb{E}[X_{P_{t}}^0]$. To estimate $\mathbb{E}[Z^0]$, consider three cases separately: i) the parent has $n-k$ 1-bits, ii) the parent has $n-k-1$ 1-bits, iii) the parent has less that $n-k-1$ 1-bits. For case i) and ii), the probability that the child has $n-k$ 1-bits is at most $\frac{k+1}{n}$ by Corollary~\ref{cor:flip_positive}. For case iii), the probability is at most ${k+2 \choose 2}(\frac{1}{n})^2\leq (\frac{k+2}{n})^2$ according to Lemma~\ref{lem:going_up}. Therefore $\mathbb{E}[Z^0]\leq \frac{k+1}{n}(c_0+c_1) + (\frac{k+2}{n})^2c(n-2k+3)=o(1)$. So by Corollary~\ref{cor:develop}, 
\[ \mathbb{E}[X_{P_{t+1}}^0] \leq (\tfrac{1}{2}+o(1))((1+\tfrac{1}{e})\mathbb{E}[X_{P_{t}}^0]+4+o(1)).\] 
As a result, $\mathbb{E}[X_{P_{t}}^0]\leq \frac{4e}{e-1}+o(1)$ for any generation $t$.
\end{proof}

Now with the upper bounds on $\mathbb{E}[X_{P_{t}}^0]$, $\mathbb{E}[X_{P_{t}}^1]$, $\mathbb{E}[X_{P_{t}}^{n-2k}]$, and $\mathbb{E}[X_{P_{t}}^{n-2k-1}]$, we can prove a lower bound on the runtime.

\begin{theorem}\label{thm:lower} Consider the NSGA-II algorithm optimizing the \onejumpzerojump benchmark with the population size $N=c(n-2k+3)$, for some $c\geq 4$ such that $ck^2=o(n)$. Then the number of fitness evaluations needed in expectation is at least $\frac{3}{2}(\frac{4}{e-1}+o(1))^{-1}Nn^k$.
\end{theorem}

\begin{proof}
Consider the waiting time to find the all-ones string when $P_0 \subseteq S_{I}^*$. For an individual of $i$ 0-bits, the probability that its child through bit-wise mutation is the all-ones string is $(\frac{1}{n})^i(1-\frac{1}{n})^{n-i}\leq (\frac{1}{n})^i$. Since by Corollary~\ref{cor:constant}, before the all-ones string is found, the expected number of individuals with $n-k$ bits of $1$ is at most $\frac{4e}{e-1}+o(1)$ in any parent population, and that for individuals with $n-k-1$ bits of $1$ is at most $\frac{e}{e-1}(c(k+2)+\frac{4e}{e-1}+4)+o(ck)$, the probability that the all-ones string is generated at any iteration is at most $(\frac{4e}{e-1}+o(1))(\frac{1}{n})^k(1-\frac{1}{n})^{n-k} + \frac{e}{e-1}(c(k+2)+\frac{4e}{e-1}+4+o(ck))(\frac{1}{n})^{k+1} + c(n-2k+3)(\frac{1}{n})^{k+2} \leq (\frac{4e}{e-1}(1-\frac{1}{n})^{n-k}+o(1))(\frac{1}{n})^k$. Then the waiting time to find the all-ones string is at least $(\frac{4e}{e-1}(1-\frac{1}{n})^{n-k}+o(1))^{-1}n^k$. Therefore, the expected number of iterations needed to find one of the all-zeroes string and the all-ones string is at least $\frac{1}{2}(\frac{4e}{e-1}(1-\frac{1}{n})^{n-k}+o(1))^{-1}n^k$ and the expected number of iterations needed to find both is at least $\frac{3}{2}(\frac{4e}{e-1}(1-\frac{1}{n})^{n-k}+o(1))^{-1}n^k$. Since by Lemma~\ref{lem:all_in}, the probability that $P_0 \subseteq S_{I}^*$ is at least $1-o(1)$, we have that the expected number of iterations needed under any case is at least $(1-o(1))(\frac{3}{2}(\frac{4e}{e-1}(1-\frac{1}{n})^{n-k}+o(1))^{-1}n^k)$, corresponding to $\frac{3}{2}(\frac{4e}{e-1}(1-\frac{1}{n})^{n-k}+o(1))^{-1}Nn^k$ fitness evaluations. Noting that $k = o(n)$, we have $(1-\frac 1n)^{n-k} = \frac 1e + o(1)$, and this completes the proof.
\end{proof}

\section{Precise Runtime of the NSGA-II with Fixed Sorting on \ojzj}
In the version of the NSGA-II considered in \cite{DoerrQ22ppsn} and the previous section, when the crowding distance is being calculated with respect to each objective, we sort the individuals such that the ones with the same objective value are positioned randomly. A variant of the algorithm, considered in \cite{BianQ22ppsn}, is to fix the relative positions of the individuals that have the same objective value. We call this variant of the algorithm the NSGA-II with fixed sorting, and show a precise bound on the runtime of this variant optimizing the \ojzj benchmark.

First we observe in the following Lemma that for this variant of the algorithm, after $O(n\log n)$ iterations, for each objective value on $F^*_I$ there are exactly two individuals with positive crowding distances.

\begin{lemma}\label{lem:fixed}
Consider the NSGA-II algorithm with fixed sorting optimizing the $\onejumpzerojump$ benchmark with population size $N=c(n-2k+3)$ for some $c\geq 2$. After $O(n\log n)$ iterations, for any generation $t$, for every objective value $v\in F_I^*$, there are exactly two individuals $x, y\in R_t$ such that $f(x)=f(y)=v$, $\cDis(x)>0$ and $\cDis(y)>0$.
\end{lemma}
\begin{proof}
We first prove that for any objective value $v=(v_1, v_2)\in F^*$, there are at most two individuals $x, y\in R_t$ such that $f(x)=f(y)=v$, $\cDis(x)>0$ and $\cDis(y)>0$. Suppose the individuals whose objective value is $v$ belong to rank $F$ of $R_t$. Let $S_{1.1},\dots,S_{1.|F|}$ be the list of individuals in $F$ sorted by ascending $f_1$ values and $S_{2.1}, \dots ,S_{2.|F|}$ be the list of individuals sorted by ascending $f_2$ values, which were used to compute the crowding distances. Then there exist $a\leq b$  and $a'\leq b'$ such that $[a..b] = \{i \mid f_1(S_{1.i})=v_1\}$ and $[a'..b'] = \{i \mid f_2(S_{2.i})=v_2\}$. For $i \in [a+1..b-1]$ and $S_{1.i}=S_{2.j}$ for some $j \in [a'+1..b'-1]$, we have that $f_1(S_{1.i-1})=v_1=f_1(S_{1.i+1})$ and $f_2(S_{2.j-1})=v_2=f_2(S_{2.j+1})$. So $\cDis(S_{1.i})= 0$. Hence the only individuals that could have positive crowding distance are $S_{1.a}$, $S_{2.a'}$, $S_{1.b}$, $S_{2.b'}$. Since the relative positions of the individuals whose objective value is $v$ are the same in $S_{1.1},\dots,S_{1.|F|}$ and $S_{2.1}, \dots ,S_{2.|F|}$, we have $S_{1.a}$ is the same individual as $S_{2.a'}$ and $S_{1.b}$ is the same individual as $S_{2.b'}$. So there are at most two individuals whose crowding distances are positive.

Next, we prove that if there are at least two individuals in $R_t$ whose objective values are $v$, then for any generation $s\geq t$, there are exactly two individuals $x, y\in R_s$ such that $f(x)=f(y)=v$, $\cDis(x)>0$ and $\cDis(y)>0$. By the definition of the crowding distance, we have that $\cDis(S_{1.a})\geq \frac{f_1(S_{1.a+1}) -  f_1(S_{1.a-1})}{f_1(S_{1.|F_1|})-f_1(S_{1.1})}\geq\frac{f_1(S_{1.a}) -  f_1(S_{1.a-1})}{f_1(S_{1.|F_1|})-f_1(S_{1.1})}$. Since $f_1(S_{1.a}) -  f_1(S_{1.a-1}) > 0$ by the definition of $a$, we have $\cDis(S_{1.a}) > 0$. Similarly, we have $\cDis(S_{1.b}) > 0$. Since there are at least two individuals $x,y\in R_t$ such that $f(x)=f(y)=v$, we have $a\neq b$, meaning $S_{1.a}$ and $S_{1.b}$ are not the same individual. Since $v\in F^*$, we have that $S_{1.a}$ and $S_{1.b}$ belong to the first rank. Consequently, these two will survive to the next generation and appear in $R_{t+1}$. Then again there will be exactly two individuals in $R_{t+1}$ whose objective values are $v$ and who have positive crowding distances. Inductively, in any generation $s\geq t$, there will always be exactly two such individuals.

By Lemma 3 of \cite{DoerrQ22ppsn}, in at most $e(\frac{4k}{3})^k$ iterations, there will be an individual $x$ in the parent population such that $f(x)\in F_I^*$. Lemma 4 of \cite{DoerrQ22ppsn} then states that once such $x$ appears, in $O(n \log n)$ iterations in expectation, every objective value $v\in F_I^*$ will have been generated at least once. Then, in $O(n\log n)$ iterations in expectation, every objective value in $F_I^*$ will have been generated at least twice. So for any generation $t$ starting from there, there will be exactly two individuals $x, y\in R_t$ such that $f(x)=f(y)=v$, $\cDis(x)>0$ and $\cDis(y)>0$.
\end{proof}

We call the phase where, for every objective value $v\in F_I^*$, there are exactly two individuals $x, y$ in the combined population such that $f(x)=f(y)=v$, $\cDis(x)>0$ and $\cDis(y)>0$, the tightening phase. In the following analyses, we define $s^*=O(n\log n)$ to be the generation where the algorithm first enters the tightening phase. Then by Lemma~\ref{lem:fixed}, for any generation $t\geq s^*$, the algorithm stays in the tightening phase. In the following Lemma, we estimate similarly to Lemma~\ref{lem:zero_survives} the probability that a rank-$1$ individual with zero crowding distance survives. What is different now is that for the tightening phase, we can calculate the probability precisely (apart from lower order terms). 

\begin{lemma}\label{lem:zero_survives_fixed}
Consider a generation $t\geq s^*$ of the NSGA-II algorithm with fixed sorting optimizing the \onejumpzerojump benchmark with population size $N=c(n-2k+3)$ for some $c\geq 2$ such that $ck^2=o(n)$. Suppose $P_0\subseteq S_I^*$, $\mathbb{E}[X_{P_t}^0]=O(ck)$, and $\mathbb{E}[X_{P_t}^{n-2k}]=O(ck)$. For a rank-1 individual $x\in R_t$ that has zero crowding distance, the probability that $x\in P_{t+1}$ is $\frac{c-2}{2c-2}\pm o(1)$.
\end{lemma}
\begin{proof}
Since now, for every objective value $v\in F_I^*$, there are exactly two individuals $x, y$ such that $f(x)=f(y)=v$, $\cDis(x) >0$ and $\cDis(y)>0$, using the same notations as in Lemma~\ref{lem:zero_survives}, we have $|F_1^*|=2(n-2k+3)-\Theta(1)$. So the probability that a rank-$1$ individual $x$ with zero crowding distance survives is \[\begin{aligned}\frac{N-|F_{1}^*|}{2N-|F_{1}^*|}&\Pr[|F_{>1}|=0]+\\&\Pr[x \text{ survives}||F_{>1}|\geq 1]]\Pr[|F_{>1}|\geq 1].\end{aligned}\] As proved in Lemma~\ref{lem:zero_survives}, $\Pr[|F_{>1}|=0]\geq 1-o(1)$. So $\Pr[\text{x survives}]=\frac{(c-2)(n-2k+3)+\Theta(1)}{(2c-2)(n-2k+3)+\Theta(1)}(1-o(1))+o(1)=\frac{c-2}{2c-2}\pm o(1)$.
\end{proof}

\begin{corollary}\label{cor:develop_fixed}
Consider a generation $t\geq s^*$ of the NSGA-II algorithm with fixed sorting optimizing the \onejumpzerojump benchmark with population size $N=c(n-2k+3)$ for some $c\geq 2$ such that $ck^2=o(n)$. Suppose $P_0\subseteq S_I^*$, $\mathbb{E}[X_{P_t}^0]=O(ck)$, and $\mathbb{E}[X_{P_t}^{n-2k}]=O(ck)$. Then for any $i\in[0..n-2k]$, we have $\mathbb{E}[X^i_{P_{t+1}}] = (\frac{c-2}{2c-2}+o(1))\mathbb{E}[X^i_{R_t}]+\frac{c}{c-1} \pm o(1)$.  
\end{corollary}
\begin{proof}
Since $N=c(n-2k+3)$ for $c\geq 2$, any individual with $n-k-i$ 1-bits and positive crowding distance survives to the next generation. Then, using the same notations as in Corollary~\ref{cor:develop}, we have $\mathbb{E}[X^i_{P_{t+1}}] = \mathbb{E}[X_{>0}] + \mathbb{E}[X_{=0}^*]$. Then by Lemma~\ref{lem:zero_survives_fixed}, \[\begin{aligned}\mathbb{E}[X^i_{P_{t+1}}] &= \mathbb{E}[X_{>0}] + (\frac{c-2}{2c-2}\pm o(1))\mathbb{E}[X_{=0}] \\&=(\frac{c-2}{2c-2}\pm o(1))\mathbb{E}[X^i_{R_t}]\\&+(\frac{c}{2c-2}\pm o(1))\mathbb{E}[X_{>0}].\end{aligned}\] Since $\mathbb{E}[X_{>0}]=2$ by Lemma~\ref{lem:fixed}, we have $\mathbb{E}[X^i_{P_{t+1}}] = (\frac{c-2}{2c-2}\pm o(1))\mathbb{E}[X^i_{R_t}]+\frac{c}{c-1}\pm o(1)$.
\end{proof}

Consequently, we can calculate $\mathbb{E}[X_{P_{t}}^0]$ and $\mathbb{E}[X_{P_{t}}^{n-2k}]$ precisely apart from lower order terms.
\begin{lemma}\label{lem:bound_fixed}
Consider the NSGA-II algorithm with fixed sorting optimizing the \onejumpzerojump benchmark with the population size $N=c(n-2k+3)$ for some $c\geq 2$ such that $ck^2=o(n)$. Suppose $P_0 \subseteq S_{I}^*$. We have for any $t\geq s^*+\log n$, if $1^n\notin P_t$, then $\mathbb{E}[X^0_{P_t}]=\frac{2ec}{ec-c+2}\pm o(1)$. Similarly, if $0^n\notin P_t$, then $\mathbb{E}[X_{P_{t}}^{n-2k}]=\frac{2ec}{ec-c+2}\pm o(1)$.
\end{lemma}
\begin{proof}
Since the NSGA-II with fixed sorting is a special case of the general algorithm considered in section~\ref{sec:low_bound_ojzj}, the bounds proven there still apply. So using the same notations as in Lemma~\ref{lem:bound}, the proof of Corollary~\ref{cor:constant} shows that $\mathbb{E}[Z^0]\le o(1)$. Also, since there are at least two individuals with $n-k-1$ 1-bits, $\mathbb{E}[Z^0]\ge 2\frac{k+1}{n}=o(1)$. So $\mathbb{E}[Z^0] = o(1)$. Moreover, $\mathbb{E}[Y^0]=(1-\frac{1}{n})^n\mathbb{E}[X^0_{P_t}]=(\frac{1}{e}-o(1))\mathbb{E}[X^0_{P_t}]$. So for any generation $t'\ge s^*$, we have \[\begin{aligned}\mathbb{E}[X^0_{P_{t'+1}}]&=(\frac{c-2}{2c-2}\pm o(1))((\frac{e+1}{e}-o(1))\mathbb{E}[X^0_{P_t'}]+o(1))\\&+\frac{c}{c-1}\pm o(1).\end{aligned}\] Hence the sequence of $\mathbb{E}[X^0_{P_t'}]$ converges to $C=\frac{2ec}{ec-c+2}\pm o(1)$. Since $\frac{|C-\mathbb{E}[X^0_{P_{t'+1}}]|}{|C-\mathbb{E}[X^0_{P_t'}]|}=\frac{(e+1)(c-2)}{2e(c-1)}\pm o(1)< 1$, we have that in $\omega(1)$ iterations, $\mathbb{E}[X^0_{P_t'}]$ reaches $\frac{2ec}{ec-c+2}\pm o(1)$. Since the sequence starts from iteration $s^*$, for any $t\geq n^2$, we have that $\mathbb{E}[X^0_{P_t}]=\frac{2ec}{ec-c+2}\pm o(1)$.
\end{proof}

\begin{theorem}\label{thm:fixed}
Consider the NSGA-II algorithm with fixed sorting optimizing the \onejumpzerojump benchmark, for $k\geq 3$, with the population size $N=c(n-2k+3)$, for some $c\geq 2$ such that $ck^2=o(n)$. Then the number of fitness evaluations needed in expectation is  $\frac{3}{2}N(\frac{2c}{ec-c+2}\pm o(1))^{-1}n^k$.
\end{theorem}
\begin{proof}
For the upper bound, if the all-ones string is not found in $O(n\log n)$ iterations, then by Lemma~\ref{lem:bound_fixed}, we have $\mathbb{E}[X^0_{P_{t}}]=\frac{2ec}{ec-c+2}\pm o(1)$ for any iteration $t$ afterwards. Then the probability of generating the all-ones string in an iteration is at least $(\frac{2ec}{ec-c+2}- o(1))(\frac{1}{n})^k(1-\frac{1}{n})^{n-k}$ and in expectation at most $(\frac{2ec}{ec-c+2}(1-\frac{1}{n})^{n-k}- o(1))^{-1}n^k$ iterations are needed to find the all-ones string. Then at most $\frac{3}{2}(\frac{2ec}{ec-c+2}(1-\frac{1}{n})^{n-k}-o(1))^{-1}n^k\le \frac{3}{2}(\frac{2c}{e
c-c+2}-o(1))^{-1}n^k$ iterations are needed to find both the all-ones and the all-zeroes string.

For the lower bound, suppose $P_0\subseteq S_I^*$. Then by Lemma~\ref{lem:all_in}, for any generation $t$ there is no individual with less than $k$ 0-bits or 1-bits in the parent population. Then for an individual $x\in P_t$ to generate the all-ones or the all-zeroes string, the probability is at most $(\frac{1}{n})^{|x|_1}(1-\frac{1}{n})^{|x|_0}+(\frac{1}{n})^{|x|_0}(1-\frac{1}{n})^{|x|_1}\leq \frac{2}{n^k}$. Then for any iteration to generate the all-ones or the all-zeroes string, the probability is at most $\frac{2N}{n^k}$. Therefore, the probability that neither the all-ones nor the all-zeroes string has been generated in $O(n\log n)$ iterations is $(1-\frac{2N}{n^k})^{O(n\log n)}\geq 1-\frac{O(cn^2\log n)}{n^k}=1-o(1)$ since $k\geq 3$ and $c=o(n)$. Then for any iteration $t$ afterwards, we have that, $\mathbb{E}[X^0_{P_{t}}]=\frac{2ec}{ec-c+2}\pm o(1)$. Moreover, by Lemma~\ref{lem:bound}, we have that $\mathbb{E}[X^1_{P_{t}}]=O(ck)$ for any generation $t$. Then the probability that the all-ones string is generated in an iteration is at most $(\frac{2ec}{ec-c+2}+ o(1))(\frac{1}{n})^k(1-\frac{1}{n})^{n-k}+O(ck)(\frac{1}{n})^{k+1}+c(n-2k+3)(\frac{1}{n})^{k+2}\leq(\frac{2ec}{ec-c+2}(1-\frac{1}{n})^{n-k}+o(1))(\frac{1}{n})^k$. So the waiting time to find the all-ones string is at least $(\frac{2ec}{ec-c+2}(1-\frac{1}{n})^{n-k}+o(1))^{-1}n^k$. Then the waiting time to find both the all-ones string and the all-zeroes strings is at least $\frac{3}{2}(\frac{2ec}{ec-c+2}(1-\frac{1}{n})^{n-k}+o(1))^{-1}n^k$. Since the probability that i) neither the all-ones string nor the all-zeroes string is found in $n^2$ generations, given that ii) $P_0\subseteq S^*_I$, and the probability that event ii) happens are both $1-o(1)$. The probability that both happen is $(1-o(1))^2$. Then the waiting time to find the two extremal points in any case is at least $(1-o(1))^2\frac{3}{2}(\frac{2ec}{ec-c+2}(1-\frac{1}{n})^{n-k}+o(1))^{-1}n^k=\frac{3}{2}(\frac{2ec}{ec-c+2}(1-\frac{1}{n})^{n-k}+o(1))^{-1}n^k=\frac{3}{2}(\frac{2c}{ec-c+2}+o(1))^{-1}n^k$ for $k=o(n)$.
\end{proof}

\section{Lower Bound on the Runtime of the NSGA-II on \oneminmax}
\cite{ZhengLD22} gave an $O(Nn\log n)$ upper bound on the runtime of the NSGA-II optimizing the \oneminmax benchmark. In this section, we prove a matching lower bound using the techniques we have developed so far.

In this section, for any $i\in[0..n]$ and any generation $t$, we let $X^i_{P_t}$ denote the number of individuals with $i$ 0-bits in $P_t$ and let $X^i_{R_t}$ denote that in $R_t$.

Suppose in an iteration $t$ the individual with the least number of $0$-bits in $P_t$ has $i_t$ $0$-bits. Then we can think of the algorithm making progress as it tries to decrease $i_t$ till it becomes $0$, at which point the algorithm has found $1^n$. In the following lemma, we give upper bounds on $\mathbb{E}[X^{i_t}_{P_t}]$ when $i_t$ is close to $0$, which will help us estimate the waiting time needed for the algorithm to make progress there.

\begin{lemma}\label{lem:oneminmax_bound}
Consider the NSGA-II algorithm optimizing the \oneminmax benchmark for $n > 16$, with $N = c(n+1)$ for $c\geq 4$, where for all $x\in P_0$, $|x|_0\geq \frac{n}{4}$. Suppose $v\in[0..n]$ such that $cv^2=o(n)$ and for a generation $t$ there is no individual $x\in P_t$ such that $|x|_0 < v$.  Then $\mathbb{E}[X^v_{P_t}]\leq c_0 = \frac{4e}{e-1}+o(1)$ and $\mathbb{E}[X^{v+1}_{P_{t}}]\leq \frac{e-1}{e}(c(v+2)+c_0 + 4)+o(c)$.
\end{lemma}

\begin{proof}
First we show that, similarly to Lemma~\ref{lem:zero_survives}, the probability that an $x\in R_t$ with zero crowding distance survives to $P_{t+1}$ is less than $\frac{1}{2}$. Unlike Lemma~\ref{lem:zero_survives}, for the \oneminmax benchmark, since any solution lies on the Pareto front, we have that $|F_{>1}|=0$. So the probability that $x$ survives is $\frac{N-|F_1^*|}{2N-|F_1^*|}<\frac{1}{2}$ without additional lower-order terms.

Then, similarly to Corollary~\ref{cor:develop} and using the same notations, $\mathbb{E}[X_{P_{t+1}}]\leq \frac{1}{2}(\mathbb{E}[X_{R_t}]+\mathbb{E}[X_{>0}])$. Since by Lemma $1$ of \cite{ZhengLD22}, $X_{>0}\leq 4$, we have $\mathbb{E}[X_{P_{t+1}}]\leq \frac{1}{2}\mathbb{E}[X_{R_t}] + 2$.

Similarly to Lemma~\ref{lem:bound}, we now prove by induction that $\mathbb{E}[X^v_{P_t}]\leq c_0 = \frac{e-1}{e}(c(v+1)+4)+o(c)$. For the base case, since for all $x\in P_0$, we have $|x|_0\geq \frac{1}{4}n > v$, because $n > 16$ and $v=o(\sqrt{n})$. So $\mathbb{E}[X^v_{P_0}]=0$. For the induction,  $\mathbb{E}[X^v_{R_t}]\leq \frac{e+1}{e}\mathbb{E}[X^v_{P_t}]+\frac{v+1}{n}c(n+1)=\frac{e+1}{e}\mathbb{E}[X^v_{P_t}]+c(v+1)+o(1)$ and in turn $\mathbb{E}[X^v_{P_{t+1}}]\leq \frac{1}{2}(\frac{e+1}{e}\mathbb{E}[X^v_{P_t}]+c(v+1)+o(1))+2$. Consequently, $\mathbb{E}[X^v_{P_t}]\leq \frac{e-1}{e}(c(v+1)+4)+o(c)=c_0$ for any generation $t$. Moreover, $\mathbb{E}[X^{v+1}_{P_{t+1}}]\leq \frac{1}{2}(\frac{e+1}{e}\mathbb{E}[X^{v+1}_{P_t}]+\frac{v+2}{n}c(n+1)+c_0)+2=\frac{1}{2}(\frac{e+1}{e}\mathbb{E}[X^{v+1}_{P_t}]+c(v+2)+o(1)+c_0)+2$. So $\mathbb{E}[X^{v+1}_{P_t}]\leq \frac{e-1}{e}(c(v+2)+c_0 + 4)+o(c)$.

Similarly to Corollary~\ref{cor:constant}, we can now prove a sharper bound on $\mathbb{E}[X^v_{P_t}]$. Now $\mathbb{E}[X^v_{P_{t+1}}]\leq \frac{1}{2}(\frac{e+1}{e}\mathbb{E}[X^v_{P_t}]+o(1))+2$, and consequently $\mathbb{E}[X^v_{P_t}]\leq \frac{4e}{e-1}+o(1)$. 
\end{proof}

Then, we show that only with very small probability, $o(n^{-\frac{4}{3}})$, the algorithm can make a progress of a size larger than $1$ in an iteration.
\begin{lemma}\label{lem:skip}
Consider the NSGA-II algorithm optimizing the \oneminmax benchmark for $n > 16$, with $N = c(n+1)$ for $c\geq 4$, where for all $x\in P_0$, $|x|_0\geq \frac{n}{4}$. Suppose $v\in[0..n]$ such that $cv^3=o(n)$ and for a generation $t$, there is no individual $x\in P_t$ such that $|x|_0 < v$. Suppose in $R_{t}$ the individual with the least number of 0-bits is $y$. Then $\Pr[|y|_0\leq v-2]=o(n^{-\frac{4}{3}})$.   
\end{lemma}
\begin{proof}
By Lemma~\ref{lem:going_up}, to generate an individual with at most $v-2$ bits of $0$ through bit-wise mutation, for an individual with $v$ bits of $0$, the probability is at most $(\frac{v}{n})^2$, for an individual with $v+1$ bits of $0$, it is at most $(\frac{v+1}{n})^3$, and for an individual with at least $v+2$ bits of $0$, it is at most $(\frac{v+2}{n})^4$. Let $X$ denote the number of individuals with at most $v-2$ 0-bits in $R_{t}$, $c_0=\frac{4e}{e-1}+o(1)$, and $c_1=\frac{e-1}{e}(c(v+2)+c_0 + 4)+o(c)$. Then by Lemma~\ref{lem:oneminmax_bound}, $\mathbb{E}[X]\leq c_0(\frac{v}{n})^2+c_1(\frac{v+1}{n})^3+c(n+1)(\frac{v+2}{n})^4=o(n^{-\frac{4}{3}})$. So, $\Pr[|y|_0\leq v-2]=\Pr[X\geq 1]\leq \frac{\mathbb{E}[X]}{1}=o(n^{-\frac{4}{3}})$.
\end{proof} 

Finally, we combine everything to obtain a lower bound on the runtime.
\begin{theorem}\label{thm:omm}
Consider the NSGA-II algorithm optimizing the \oneminmax benchmark for $n > 16$, with $N = c(n+1)$ for $c\geq 4$ and $c=o(n^\mu)$ for $\mu < 1$. Then the number of fitness evaluations needed is at least $N(n(\frac{4e}{e-1}+o(1))^{-1}\frac{1-\mu}{3}\ln n)$.
\end{theorem}
\begin{proof}
Let $v=n^\frac{1-\mu}{3}$ ($cv^3=o(n)$). For any generation $t$, suppose $x^t$ is the individual with the least number of $0$-bits in $P_t$, and we denote $|x^t|_0$ by $M_t$. Suppose we have $M_0\geq v$, which happens with probability at least $1-Ne^{-\frac{n}{8}}=1-o(1)$ as proven in Lemma~\ref{lem:all_in}, since $v<\frac{n}{4}$. Then suppose at generation $r$, we have $M_r=v$. Then for any generation $s\geq r$, we have $c(M_s)^3=o(n)$ and by Corollary~\ref{lem:skip}, the probability that $M_s-M_{s+1}\geq 2$ is $o(n^{-\frac{4}{3}})$. Since by Theorem $6$ of \cite{ZhengLD22}, the expected iterations needed to optimize the \oneminmax benchmark is $O(n\log n)$, we have that the expected times that the event $M_s-M_{s+1}\geq 2$ happens for any $s\geq r$ is $o(n^{-\frac{4}{3}})O(n\log n)=o(n^{-\frac{1}{3}}\log n)=o(1)$. Then the probability that the event $M_s-M_{s+1}\geq 2$ does not happen for any $s\geq r$ is $1-o(1)$. When that happens, after generation $r$, the algorithm makes progress by decreasing $M_s$ one by one until $M_{s'}=0$ for some $s'$.

By Lemma~\ref{lem:going_up}, to generate an individual with $M_s-1$ bits of $0$ through bit-wise mutation, for an individual with $M_s$ bits of $0$, the probability is at most $\frac{M_s}{n}$, for an individual with $M_s+1$ bits of $0$, the probability is at most $(\frac{M_s+1}{n})^2$, and for an individual with less than $M_s+1$ bits of $0$, the probability is at most $(\frac{M_s+2}{n})^3$. Let $X$ denote the number of individuals with $M_s-1$ bits of $0$ in $R_s$, $c_0=\frac{4e}{e-1}+o(1)$, and $c_1=\frac{e-1}{e}(c(v+2)+c_0 + 4)+o(c)$. Then by Lemma~\ref{lem:oneminmax_bound}, $\mathbb{E}[X]\leq c_0\frac{M_s}{n}+c_1(\frac{M_s+1}{n})^2+c(n+1)(\frac{M_s+2}{n})^3=(\frac{4e}{e-1}+o(1))\frac{M_s}{n}$. So the probability that $X\geq 1$ is at most $(\frac{4e}{e-1}+o(1))\frac{M_s}{n}$ and the waiting time for the event $M_{s}-M_{s+1}=1$ to happen is at least $(\frac{4e}{e-1}+o(1))^{-1}\frac{n}{M_s}$ iterations. So, for $M_s$ to decrease from $v$ to $0$ one by one, the total waiting time is $\sum_{i=1}^v(\frac{4e}{e-1}+o(1))^{-1}\frac{n}{i}>n(\frac{4e}{e-1}+o(1))^{-1}\ln v=n(\frac{4e}{e-1}+o(1))^{-1}\frac{1-\mu}{3}\ln n$.

Since the probability that event i) once $M_s$ reaches $v$, it decreases by at most one in an iteration happen given that event ii) $M_0\geq v$ happen, and the probability that event ii) happens, are both at least $1-o(1)$, and when both i) and ii) happen, the number of iterations needed is at least $n(\frac{4e}{e-1}+o(1))^{-1}\frac{1-\mu}{3}\ln n$, we have that in any case, the number of iterations needed is at least $(1-o(1))^2(n(\frac{4e}{e-1}+o(1))^{-1}\frac{1-\mu}{3}\ln n)=(1-o(1))(n(\frac{4e}{e-1}+o(1))^{-1}\frac{1-\mu}{3}\ln n)$, corresponding to $N(n(\frac{4e}{e-1}+o(1))^{-1}\frac{1-\mu}{3}\ln n)$ fitness evaluations.
\end{proof}

\section{Experiments}

To complement our theoretical results, we also experimentally evaluate some runs of the NSGA\nobreakdash-II on the \ojzj benchmark, both with respect to the runtime and the population dynamics. We note that \cite{DoerrQ22ppsn} already presented some results on the runtime (for $k=3$, $N/(2n-k+3) = 2, 4, 8$, and $n=20, 30$). We therefore mostly concentrate on the population dynamics, i.e., the number of individuals in the population for each objective value on the Pareto front, which our theoretical analyses have shown to be crucial for determining the lower bound and the leading coefficient of the runtime.

\subsection{Settings}

We implemented the algorithm as described in the Preliminaries section in Python, and tested the following settings.
\begin{itemize}
    \item Problem size $n$: $50$ and $100$.
    \item Jump size $k$: $2$. This  small number was necessary to admit the problem sizes above. Problem sizes of a certain magnitude are needed to see a behavior not dominated by lower-order effects. 
    \item Population size $N$: $2(n-2k+3)$, $4(n-2k+3)$, and $8(n-2k+3)$. \cite{DoerrQ22ppsn} suggested that, even though their mathematical analysis applies only for $N\geq 4(n-2k+3)$, already for $N=2(n-2k+3)$ the algorithm still succeeds empirically. Therefore, we have also experimented with $N=2(n-2k+3)$ to confirm that our arguments for the population dynamics still apply for the smaller population size.
    \item Selection for variation: fair selection.
    \item Mutation method: bit-wise mutation with rate $\frac{1}{n}$.
    \item Number of independent repetitions per setting: $10$.
\end{itemize}

\subsection{Results on the Runtime}

\begin{table}[t]
\centering
\begin{tabular}{|r|r|r|}
\hline
& $n=50$ & $n=100$ \tabularnewline
\hline
$N=2(n-2k+3)$ & $247{,}617$ & $2{,}411{,}383$ \\
$N=4(n-2k+3)$ & $416{,}284$ & $3{,}858{,}084$ \\
$N=8(n-2k+3)$ & $714{,}812$ & $6{,}792{,}456$ \\
\hline
\end{tabular}
\caption{Average runtime of the NSGA-II with bit-wise mutation on the \ojzj benchmark with $k=2$.}\label{tab1}
\end{table}

Table~\ref{tab1} contains the average runtime (number of fitness evaluations done until the full Pareto front is covered) of the algorithm. For all of the settings, we have observed a standard deviation that is between $50\%$ to $80\%$ of the mean, which supports our reasoning that the runtime is dominated by the waiting time needed to find the two extremal points of the front, which is the maximum of two geometric random variables. An obvious observation from the data is that increasing $N$ does not help with the runtime, supporting our theoretical results that the lower bound on the runtime increases when $N$ increases. Moreover, for all the settings that we have experimented with, the average runtime is well above our theoretically proven lower bound made tighter by discarding the lower order terms, namely $\frac{3}{2}(\frac{4}{e-1})^{-1}Nn^k$.

\subsection{Results on the Population Dynamics}

\begin{figure}[t]
\centering
\includegraphics[width=0.9\columnwidth]{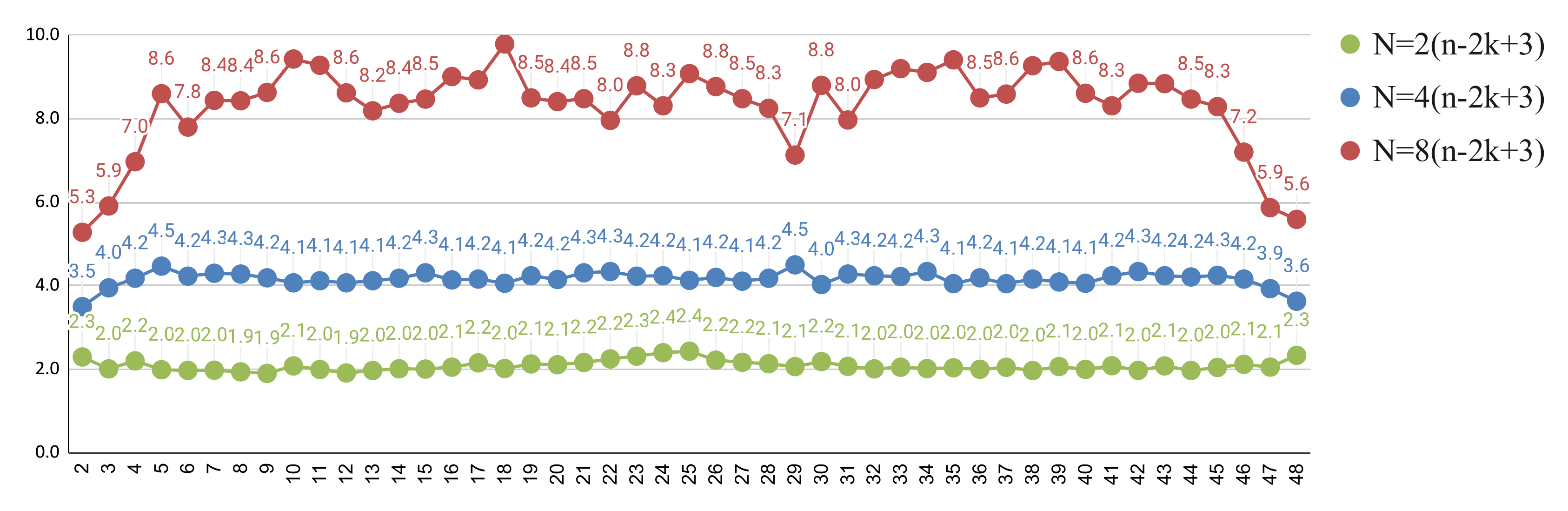}
\caption{Average number of individuals with $i\in [k .. n-k]$ $1$-bits for $n=50$ and $k=2$.}
\label{fig1}
\end{figure}
\begin{figure}[t]
\centering
\includegraphics[width=0.9\columnwidth]{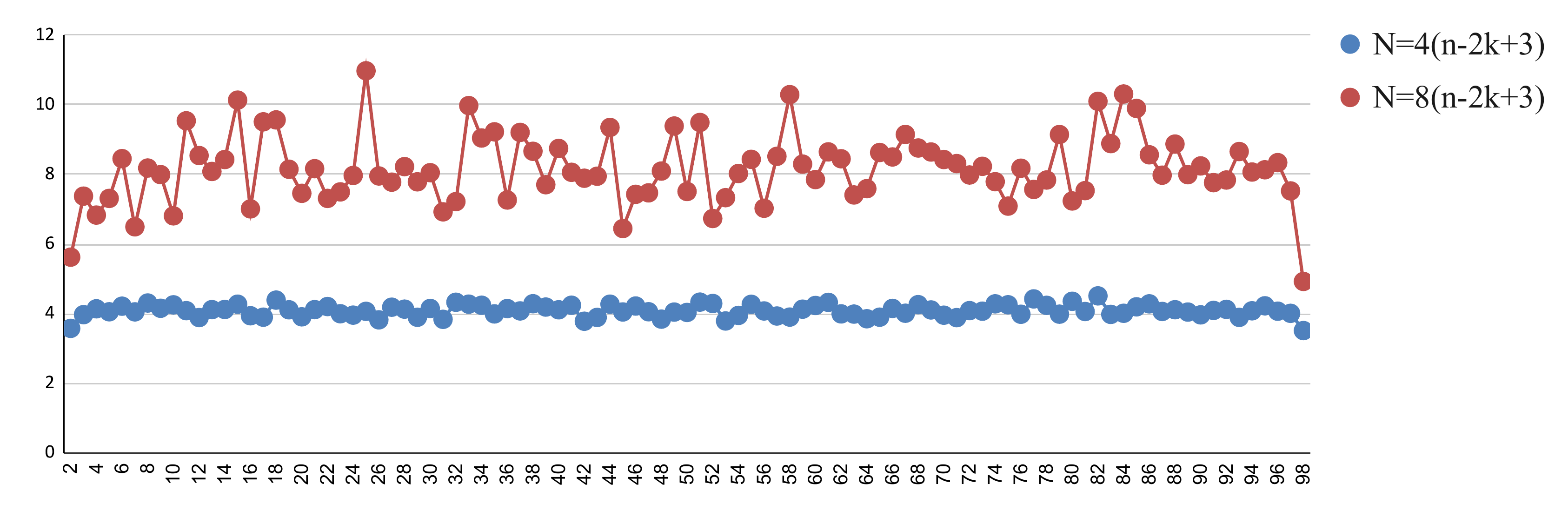}
\caption{Average number of individuals with $i\in [k .. n-k]$ $1$-bits for $n=100$ and $k=2$.}
\label{fig2}
\end{figure}

For most of the experiments conducted, we have also recorded the population dynamics throughout the executions of the algorithm (mistakenly, we did not do so for $n=100$ and $N = 2(n-2k+3)$; however, from the very even distribution seen for $n=50$ for this population size, we would be very surprised to see a different pattern for $n=100$; of course, we will fix this omission for the final version of this paper). Specifically, for each run, for every $n^k/50$ iterations, we record for each $i\in [k .. n-k]$ how many individuals there are in the parent population with $i$ bits of $1$. Since as shown in our theoretical analyses and \cite{DoerrQ22ppsn}, the greatest contributor to the runtime is the waiting time to find the all-ones and the all-zeroes strings after the inner part of the Pareto front has been discovered, we are mostly interested in how the population dynamics develop in that phase. To this end, we discard data points recorded when the inner part of the Pareto front has not been fully covered, and those recorded after one of the extremal points has already been discovered. The final numbers reported for a run is the average of the data points kept. In the end we report the average of the means across $10$ repetitions. In all of the runs, we have never observed an initial population not contained in the inner part of the Pareto front, supporting our theoretical arguments and also making the experiments fall into the scenario that we have studied theoretically.

Figure \ref{fig1} contains the average number of individuals throughout a run of the algorithm for each point on the inner part of the Pareto front for $n=50$, averaged by the $10$ repetitions, and Figure \ref{fig2} contains that for $n=100$. An obvious observation is that for all experiment settings, we have that the average number of individuals with $k$ or $n-k$ $1$-bits is less than the proven upper bound $\frac{4e}{e-1} \approx 6.33$. When doubling the population size, the number of individuals with $k$ or $n-k$ $1$-bits grows. This does not contradict with our upper bound (which is independent of the population size), but it only suggests that the precise average occupation of these objective values contains a dependence on the population size that is small enough for this number to be bounded by $\frac{4e}{e-1} \approx 6.33$. We note that for the setting with fixed sorting the precise occupation number $\frac{2ec}{2ec-c+2} \pm o(1)$ we proved displayed exactly such a behavior. 

Our experimental data also give the occupation numbers for the other objective values. We did not discuss these in much detail in our theoretical analysis since all we needed to know was the occupation number for the outermost points of the inner part of the Pareto front and a relatively generous upper bound for the points one step closer to the middle. A closer look into our mathematical analysis shows that it does give good estimates only for objective values close to the outermost points of the Pareto front. For that reason, it is interesting to observe that our experimental data show that the population is, apart from few positions close to the outermost positions, very evenly distributed on the Pareto front (that is, a typical position is occupied by $c$ individuals, where $c$ is such that the population size is $N=c(n-2k+3)$). Given the mostly random selection of most of the next population (apart from the up to $4$ individuals with positive crowding distance per position) and the drift towards the middle in the offspring generation (e.g., a parent with $\frac 34 n$ ones is much more likely to generate an offspring with fewer than more ones), this balanced distribution was a surprise to us. While it has no influence on the time to find the Pareto front of \ojzj, we suspect that such balanced distributions are preferable for many other problems.

\section{Conclusions and Future Works}

In this work, we gave the first lower bounds matching previously proven upper bounds for the runtime of the \NSGA. We proved that the runtime of the \NSGA with population size at least four times the Pareto front size computes the full Pareto front of the \oneminmax problem in expected time (number of function evaluations) $\Omega(N n \log n)$ and the one of the \ojzj problem with jump size $k$ in expected time $\Omega(N n^k)$. These bounds match the corresponding $O(N n \log n)$ and $O(N n^k)$ upper bounds shown respectively in~\cite{ZhengLD22} and \cite{DoerrQ22ppsn}. These asymptotically tight runtimes show that, different from many other population-based search heuristics, the \NSGA does not profit from larger population sizes, even in an implementation where the expected numbers $\Theta(n \log n)$ and $\Theta(n^k)$ of iterations is the more appropriate performance criterion. Together with the previous result~\cite{ZhengLD22} that a population size below a certain value leads to a detrimental performance of the \NSGA, our results show that the right choice of the population size of the \NSGA is important for an optimal performance, much more than for many single-objective population-based algorithms, where larger population sizes at least for certain parameter ranges have little influence on the number of fitness evaluations needed.

The main obstacle we had to overcome in our analysis was to understand sufficiently well the population dynamics of the \NSGA, that is, the expected number of individuals having a particular objective value at a particular time. While we have not completely understood this question, our estimates are strong enough to obtain, for the \ojzj benchmark and the \NSGA using a fixed sorting to determine the crowding distance, a runtime guarantee that is also tight including the leading constant.

From this work, a number of possible continuations exist. For example, runtime analyses which are tight including the leading constant allow one to distinguish constant-factor performance differences. This can be used to optimize parameters or decide between different operators. For example, we have used the mutation rate $\frac 1n$, which is the most accepted choice for bit-wise mutation. By conducting our analysis for a general mutation rate $\frac \alpha n$, one would learn how the mutation rate influences the runtime and one would be able to determine an optimal value for this parameter. We note that a different mutation rate not only changes the probability to reach the global optimum from the local one (which is well-understood~\cite{DoerrLMN17}), but also changes the population dynamics. We are nevertheless optimistic that our methods can be extended in such directions. 

\bibliographystyle{alphaurl}
\bibliography{ich_master,alles_ea_master}
}
\end{document}